\pdfoutput=1

\documentclass[11pt]{article}

\usepackage{ACL2023}

\usepackage{times}
\usepackage{latexsym}

\usepackage[T1]{fontenc}

\usepackage[utf8]{inputenc}

\usepackage{microtype}

\usepackage{amsfonts}       
\usepackage{amsmath}
\usepackage{amsthm}
\usepackage{amssymb}
\usepackage{nicefrac}       
\usepackage{xfrac}
\usepackage{subcaption}
\usepackage{multirow}
\usepackage{ctable} 
\usepackage{booktabs}       
\usepackage{diagbox}
\usepackage{csquotes}
\usepackage{enumitem}
\usepackage{inconsolata}
\usepackage{lipsum}
\usepackage{abraces}

\newtheorem{theorem}{Theorem}[section]
\newtheorem*{claim}{Claim}

\newtheorem{defn}{Definition}[section]

\title{Learning to Generate Equitable Text in Dialogue \\ from Biased Training Data}

\author{Anthony Sicilia \\
  Intelligent Systems Program \\
  University of Pittsburgh \\
  \texttt{anthonysicilia@pitt.edu} \\\And
  Malihe Alikhani \\
  School of Computing and Information\\
  University of Pittsburgh \\
  \texttt{malihe@pitt.edu} \\}

\begin{document}
\maketitle
\begin{abstract}
The ingrained principles of fairness in a dialogue system's decision-making process and generated responses are crucial for user engagement, satisfaction, and task achievement. Absence of equitable and inclusive principles can hinder the formation of common ground, which in turn negatively impacts the overall performance of the system.
For example, misusing pronouns in a user interaction may cause ambiguity about the intended subject. Yet, there is no comprehensive study of equitable text generation in dialogue. Aptly, in this work, we use theories of computational learning to study this problem. We provide formal definitions of equity in text generation, and further, prove formal connections between learning human-likeness and learning equity: algorithms for improving equity ultimately reduce to algorithms for improving human-likeness (on augmented data). With this insight, we also formulate reasonable conditions under which text generation algorithms can learn to generate equitable text without any modifications to the biased training data on which they learn. To exemplify our theory in practice, we look at a group of algorithms for the GuessWhat?! visual dialogue game and, using this example, test our theory empirically. Our theory accurately predicts relative-performance of multiple algorithms in generating equitable text as measured by both human and automated evaluation. 
\end{abstract}

\section{Introduction}
\label{sec:intro}
Machine learning models for text generation in dialogue have trouble learning the ``long tail'' of a data distribution; i.e., the data concepts not frequently observed during training. For example, dataset biases like gender imbalance can induce a long tail in training data whereby important data relationships involving gender are underrepresented, like women in sports \citep{hendricks2018women}. When training, generative models often fail to learn these concepts in the long tail, and ultimately, learn inequitable, stereotyping behaviors instead (see Figure~\ref{fig:problem}). These non-inclusive behaviors not only decrease user-satisfaction by isolating users \citep{mehrabi2021survey}, but also impede common ground, hindering the task-success of the dialogue system. 
\begin{figure}
    \centering
    \includegraphics[width=.95\columnwidth]{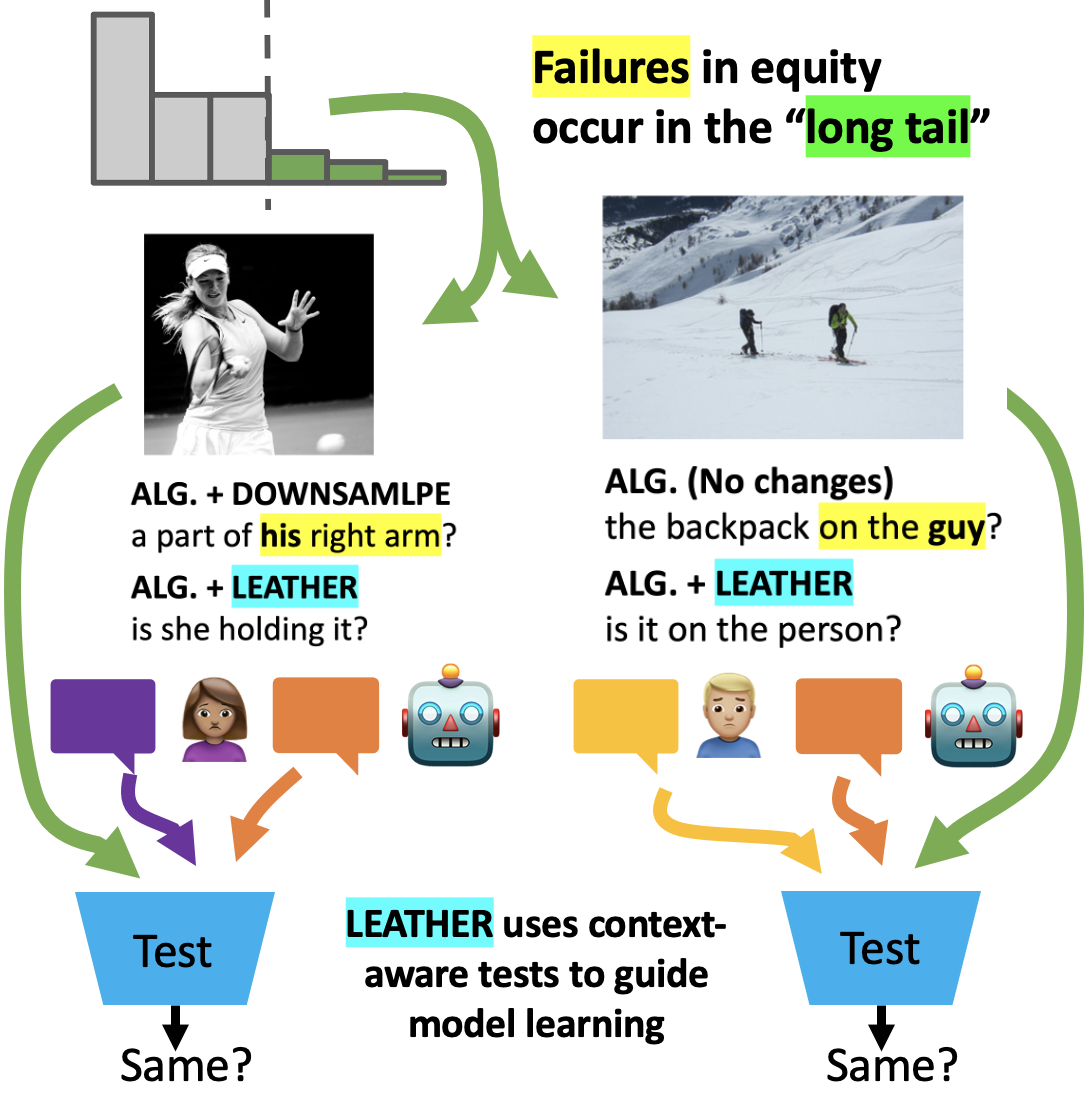}
    \caption{Examples from \textit{GuessWhat?!} dataset, which consists of game dialogues where a question-player queries an answer-player to identify a secret goal object known only to the answer-player. Images containing men outnumber images with women 2 to 1, forming a "long tail" in the data distribution. In examples above, human annotations use visual (contextual) cues to agree on gender or balk when appropriate. Traditional algorithms can be incorrect and overconfident, inheriting dataset bias towards male pronouns. Algorithms motivated by learning theory ($\texttt{LEATHER}$) are more robust, utilizing context in a human-like way.}
    \label{fig:problem}
\end{figure}

Despite the multi-faceted impact of inequitable text generation in dialogue, we do not have a comprehensive and theoretically grounded framework for understanding how machines learn to generate inequitable text and when this outcome can be avoided. To provide a strong technical foundation for equitable generation in dialogue, we build on theories of computational learning \citep{valiant1984theory, mcallester1998some}.
Specifically, our theoretical contributions are as follows:
\begin{enumerate}[leftmargin=*,nolistsep]
    \item We define precise constraints that encapsulate diverse notions of equity in dialogue (Def.~\ref{defn:sp}).
    \item We rigorously compare our proposals to traditional notions of equity in classification (\S~\ref{sec:theory_defn}). 
    \item We show computational learning theory models equitable learning well: algorithms from learning theory are easily adapted to learn equitable dialogue by augmenting data (Thm.~\ref{thm:sp_reduction}).
    \item We prove algorithms based on learning theory can even learn to generate equitable text from some types of biased training data (Thm.~\ref{thm:learning_bound}).
\end{enumerate}

Loosely, Thm.~\ref{thm:learning_bound} is based on the idea that, when provided sufficient background, human text is \textit{not} biased because it is typically \textit{context-aware} (Def.~\ref{defn:c-awar}). For example, when the subject is a female scientist, a human will likely \textit{not} use male pronouns in subject-referring conversation because humans tend to correctly employ dialogue context to inform their language use.
Instead, in many real-world datasets, bias is an \textit{aggregate property}, arising from inequality of the proportions of protected attributes such as race or gender; e.g., more conversations about male than female doctors. 

The theoretical understanding we contribute is imperative because it informs algorithm design. In particular, using our theory, we can predict: 
\begin{enumerate}[leftmargin=*,nolistsep]
    \item the most equitable algorithms for unseen data;
    \item counter-intuitive properties of algorithms that lead to less equitable results.
\end{enumerate}
For example, consider algorithms which na\"ively augment data to remove bias \citep{zhao-etal-2018-gender, park-etal-2018-reducing}. Through theoretical study, we identify cases where this practice can actually \textit{hurt} an algorithm's chances at learning to be equitable. In fact, our experiments in \S~\ref{sec:experiments} confirm this.

The remainder of the paper is organized as follows: \S~\ref{sec:background} provides background to position our contributions including discussion of related work, a brief tutorial on the employed learning theoretic framework, and a few running examples used throughout the text; \S~\ref{sec:theory} provides our theoretical contributions including formulation of mathematical notions of equity in text generation and theoretical analysis of learning algorithms; \S~\ref{sec:experiments} conducts experiments which validate our theory in practice; and finally, \S~\ref{sec:conclusion} concludes the work. Code, data, and a python package will be made publicly available to promote further research.\footnote{\href{https://github.com/anthonysicilia/equitable-dialogue-ACL2023}{ https://github.com/anthonysicilia/equitable-dialogue-ACL2023}}

\section{Background and Related Work}
\label{sec:background}
\begin{figure*}
    \centering
    \includegraphics[width=.95\textwidth]{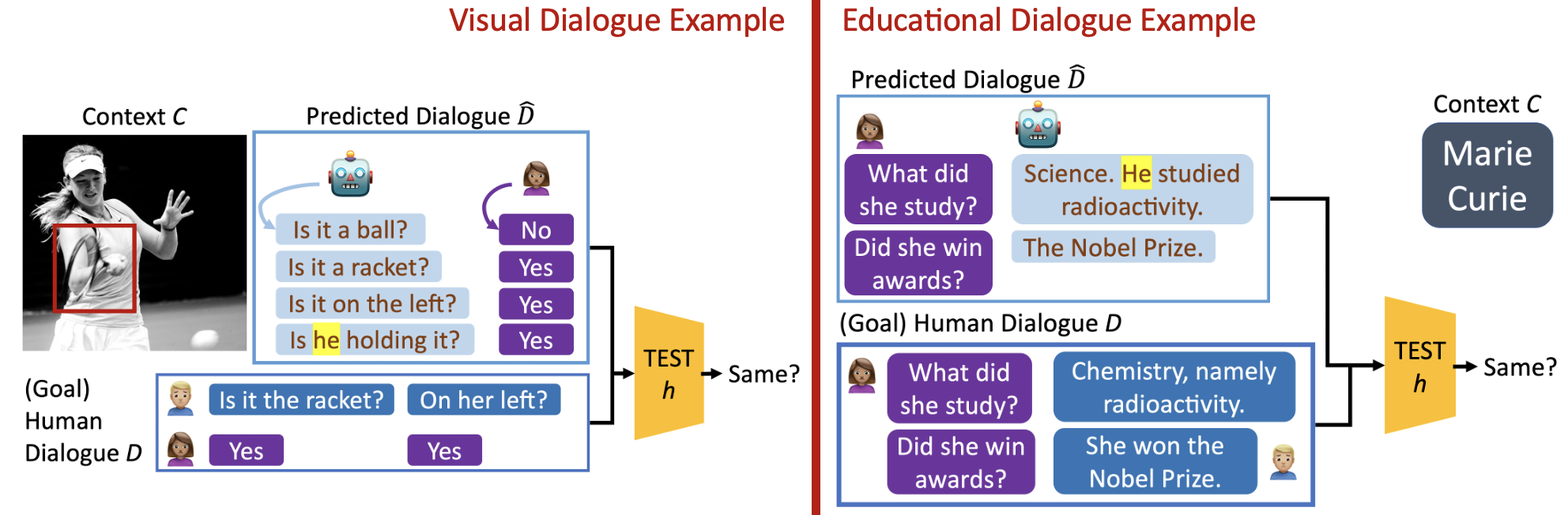}
    \caption{Toy examples of visual (left) and educational (right) dialogues. Key learning-theoretic terms are labeled. A common inequity of dialogue systems is pictured (i.e., misidentifying gender).}
    \label{fig:examples}
\end{figure*}
\subsection{Learning Theory for Dialogue}
\label{sec:bg_leather}
Recent proposals for the use of learning theory in dialogue are due to \citet{leather} who propose \texttt{LEATHER}.\footnote{\texttt{Lea}rning \texttt{Th}eory for Text-Gen\texttt{er}ation} Specifically, \texttt{LEATHER} is a formal framework for studying the diverse objectives present when learning to generate text. Ultimately, their proposal is grounded in a general evaluation metric -- the \textbf{test divergence}. Intuitively, test divergence mimics practical evaluation, in which we conduct tests to evaluate the generated dialouge:%
\begin{equation}\small
\label{eqn:TD}
\begin{split}
    & \mathbf{TD}_\mathbb{G}(\theta) = \mathbf{E}[\lvert h(D, U) - h(\hat{D}, U) \rvert] \\ 
    & \text{where} \quad (C, D) \sim \mathbb{G}, \ \hat{D} \sim \mathbb{P}_\theta(C), \  U \sim \mathbb{U}.
\end{split}
\end{equation}%
Of course, there are a number of undefined terms here: specifically, the \textit{test} $h$, the \textit{context} $C$, the \textit{goal dialogue} $D$, the \textit{learned dialogue} $\hat{D}$, and the \textit{unobserved effects} $U$. Below, we explain each, using examples from Figure~\ref{fig:examples} to assist our exposition.
\paragraph{Goal Distribution} The \textbf{goal distribution} $\mathbb{G}$ is a joint probability distribution over dialogue contexts $c \in \mathcal{C}$ and dialogues $d \in \mathcal{D}$. For \citet{leather}, the \textit{goal} is to generate human-like text. So, as in the visual dialogue example in Figure~\ref{fig:examples}, the context might be an image/goal-object and the goal dialogue might be sampled from a (human) corpus of QA pairs with this context.
\paragraph{Learned Dialogue Distribution} The \textbf{learned dialogue distribution} is the probability kernel $\mathbb{P}_\theta(C)$ that provides a distribution over dialogues, conditional to the parameters $\theta$ learned by the machine (e.g., neural parameters) as well as the random dialogue context $C$. The precise manner in which dialogue occurs will vary from system to system, but typically involves a machine generating/prompting responses to/from human users as in Figure~\ref{fig:examples}. This interaction implicitly defines the random process through which a set of parameters $\theta$ and a random context $C$ produce a predicted dialogue $\hat{D}$. Importantly, the learning machine may not control every aspect of the process -- e.g., the human responses. Aptly, we encapsulate this unknown randomness by the distribution $\mathbb{P}_\theta(C)$. In some cases, we will consider the joint distribution of both (goal) contexts and learned dialogues; i.e., of the random tuple $(C, \hat{D})$. We write $\hat{\mathbb{G}}_\theta$ for this joint distribution.
\paragraph{Test Function with Unknown Effects} The final component is the \textbf{test function} (or simply \textit{test}) $h$. The test takes as its primary input a dialogue and returns a value in the interval $[0,1]$. Conceptually, a test can represent any evaluation process in which we are interested. For example, some tests commonly employed in practice include $n$-gram overlap metrics such as BLEU \citep{papineni2002bleu}, sentiment scores from a pre-trained classifier, or even a score attained through human evaluation. The \textit{unknown effect} $U \sim \mathbb{U}$ represents \textit{any} additional information needed to completely determine the outcome of the test. When the test is BLEU, $U$ simply takes the form of a reference dialogue to which the input dialogue is compared. For human evaluation, $U$ encapsulates all of the unknown variables that contribute to the randomness of a real-world experiment. Often, $U$ may not be needed. 
\paragraph{Interpretation} With terms defined, it is easy to see the test divergence is a direct comparison of the output of the test from the goal dialogue $D$ to the predicted dialogue $\hat{D}$, learned by our dialogue system. Larger test divergence indicates the learned dialogue fails to replicate the goal dialogue along the dimensions targeted by the test. For example, if the goal is human-likeness in the visual dialogue example from Figure~\ref{fig:examples}, a test might target question strategies \citep{shekhar-etal-2019-beyond}. 
Small test divergence in these cases indicates the learned dialogue uses similar strategies as the (human) goal.
\subsection{Related Works on Equity}
\label{sec:bg_related_works}
In natural language, popular, early studies of equity begin with avoiding stereotyping in learned model representations \citep{bolukbasi2016man}. This approach has continued to inspire many de-biasing techniques for learned representations \citep{zhao-etal-2018-learning, madras2018learning, wang-etal-2020-double} and evaluation techniques for the equity of representations \citep{caliskan2017semantics, ethayarajh2019understanding}. De-biasing and evaluation techniques for model representations have also been adapted for text-generation tasks \citep{escude-font-costa-jussa-2019-equalizing, yeo-chen-2020-defining, guo-etal-2022-auto}. 

Still, these model-intrinsic approaches to resolving inequity have proven subpar compared to model-extrinsic approaches, which focus directly on the downstream task \citep{gonen-goldberg-2019-lipstick, cao-etal-2022-intrinsic}. For this reason, our approach tackles the problem of equitable dialogue generation from an extrinsic point-of-view. Previously, in text-generation, extrinsic points-of-view have typically used change in scoring functions (e.g., for sentiment, gender-polarity, etc.) to measure equity \citep{liu2020does, vu-etal-2020-multimodal, dhamala2021bold, dhamala2022analysis, das2022quantifying}. Our work is in line with these, but provides formal theoretical study, and further, focuses more specifically on dialogue. Formal theoretical study is vital to understanding equity, because imprecision in problem assumptions and objectives has already proven to be a pitfall in existing works on equity \citep{blodgett-etal-2021-stereotyping}. For example, in classification, detailed theoretical study reveals a complex relationship of trade-offs between accuracy and (some) notions of equity \citep{zhao2019inherent, mcnamara2019costs, dutta2020there}, contributing to algorithmic advances \citep{zhao2019conditional}. Our work continues this trajectory, offering valuable practical insights, which are sometimes unintuitive, to achieve equity in machine dialogue. 

Finally, it is worthwhile to note that \citet{liu2020does} also contribute a formal, theoretical definition of fairness in dialogue. Our work contributes a more general definition of equity -- i.e., which supports arbitrary types of dialogue context and more general types of dataset bias. As noted, we also make connections with learning theory to provide key insights on algorithm and dataset design. Indeed, ours is the first work to study bias in text generation using these insightful techniques from computational learning theory.

\section{Formalizing Equity in Dialogue}
\label{sec:theory}
\subsection{Formal Definitions for Equity}
\label{sec:theory_defn}
In this part, we introduce some formal, mathematical notions of equity. We start with a general notion of equity in dialogue and show how this can be specialized to compare with ideas of equity in the classification literature. For proofs, see Appendix~\ref{sec:proofs}.
\paragraph{Protected Attributes} To begin, we need to first define the notion of a \textbf{protected attribute}. Conceptually, this is the sensitive variable (e.g., race, gender, religion, etc.) that we intend to ``protect'' by the equity constraint. Otherwise, presumably, system inequities would disproportionately, negatively impact the sub-population captured by the attribute. Throughout this work, we use a variable $a \in \mathcal{A} = \{0,1\}$ to denote the protected attribute and we measure equity of the text with respect to this variable. Precisely, $a=1$ implies the dialogue context exhibits the attribute (e.g., female gender, Black race, Muslim religion), while $a=0$ implies the context does not exhibit the protected attribute. For example, in the educational dialogue from Figure~\ref{fig:examples}, the context is a discussion topic and the protected attribute is female gender. Since the topic is a female scientist, it exhibits the protected attribute and we would have $a=1$. If the topic was ``Science'' more generally, it would not exhibit the protected attribute and it would be appropriate to set $a=0$. In general, we expect the protected attribute to vary \textit{randomly} with the dialogue context $C$. To model this in a general way, we assume the attribute is sampled from a probability distribution which is dependent on the random context: $A \sim \mathbb{A}(C)$.
For example, in the visual dialogue from Figure~\ref{fig:examples}, the protected attribute $A$ is female gender, which is non-deterministically dependent on the visual features of the image $C$. In other cases, like the educational example, the protected attribute may be completely determined by context. $\mathbb{A}$ can model this as well -- e.g., as a point mass.
\paragraph{Equity as Score Parity}
Commonly, equity in machine learning systems is formally defined through a notion of \textit{parity} \citep{kamiran2009classifying, zhao2019inherent}. In dialogue, we can express parity as the following requirement:
\begin{displayquote}
\textit{The system uses language in the same way, regardless of protected attribute.}
\end{displayquote}
This intuitive notion of equity is vague in its use of ``way'' to be general, allowing for specification to different applications. For example, \citet{das2022quantifying, dhamala2022analysis} both consider the \textit{toxicity} and \textit{sentiment} of language as the pertinent ``way'' in which language is used, when measuring equity. A classifier is used to estimate the toxicity or sentiment of the used language, and equity occurs if this classifier's outputs are invariant of the protected attribute. For example, if the protected attribute is Muslim religion, the dialogue should be no more ``toxic'' when its context is specific to Muslims, than when its context is not specific to Muslims. Below, we formalize this intuition for equity with a mathematical constraint.
\begin{defn}
\label{defn:sp}
(Score Parity) A contextualized dialogue distribution\footnote{Frequently, we use \textit{contextualized dialogue distribution} to refer to any joint distribution over contexts and dialogues.} $\mathbb{G}$ with $(C,D) \sim \mathbb{G}$ and $A \sim \mathbb{A}(C)$ satisfies \textbf{score parity} if
\begin{equation}\small
    \mathbf{E}[s(D, 0) \mid A = 0] = \mathbf{E}[s(D, 1) \mid A = 1]
\end{equation}
where $s$ is a scoring function $s : \mathcal{D} \times \mathcal{A} \to [0,1]$.
\end{defn}
To arrive at our motivating example \citep{das2022quantifying, dhamala2022analysis}, one simply chooses the scoring function $s$ to be a toxicity classifier or a sentiment classifier. The expected output of this classifier should be the same, regardless of the protected attribute's setting. In general, if equality does not hold in the above definition of parity, we follow \citet{zhao2019inherent} using $\Delta$ to denote the gap across attributes:%
\begin{equation}\small
\begin{split}
    \Delta(\mathbb{G}) = \lvert  \mathbf{E}[s(D,0)\mid A=0]
     - \mathbf{E}[s(D,1)\mid A=1] \rvert.
\end{split}
\end{equation}%
This lets us talk about degrees of inequity, and therefore, measure progress towards our ideals.
\paragraph{Multi-Category Score Parity}
Notice, we use the presence/absence of singular demographic groups (e.g., \textit{female} v. \textit{not female}) instead of binary comparisons (e.g., \textit{female} v. \textit{male}) in defining the protected attribute. This choice allows our definition of equity (above) and later theory to support study of general multi-category attributes with more than two attributes like race (e.g., Black, White, Asian) or religion (e.g., Muslim, Jewish, Catholic). 
Using race as an example, we can measure the parity gap when \textit{Black} is the protected attribute, \textit{White} is the protected attribute, \textit{Asian} is the protected attribute, etc. The dataset is then equitable for all races (according to score parity) if all measured parity gaps are 0. In this way, our definition and subsequent results can generalize to the multi-category case. We use this strategy, for example, in Section~\ref{sec:experiments}.  
\paragraph{Comparison to Demographic Parity}
In classification, \textit{demographic parity} is a commonly studied notion of equity \citep{kamiran2009classifying, calders2009building, zemel2013learning}, which stipulates that a classifier's outputs should be independent of the protected attribute. For a classifier $c$, mapping random features $X$ to a $\{0,1\}$-valued label, this can be written:%
\begin{equation}\small
\label{eqn:dp}
    \mathbf{E}[c(X) \mid A = 0] = \mathbf{E}[c(X) \mid A = 1].
\end{equation}%
For score parity, when $s(\cdot, 0) = s(\cdot, 1)$, the scoring function $s$ does not depend on the attribute and we see that score parity is a direct reflection of demographic parity. Whereas classification problems use machine learning to select the classifier $c$ in a fair way, dialogue uses machine learning to select the feature distribution $X$ (i.e., $D$ in our definition). 
\paragraph{Comparison to Accuracy Parity}
Depending on the application, it is known that demographic parity can also be an inappropriate constraint; e.g., if the classifier $c$ is meant to predict the protected attribute itself \citep{zhao2019inherent}. This precise situation is inherent to dialogue, since some aspects of language are compulsorily predictive of the protected attribute (e.g., gendered pronouns or religious terminology). Fundamentally, there is a trade off between the accuracy of the language used and the desired invariance. In these cases, \citet{zhao2019inherent} suggest \textit{accuracy parity} as an alternative, which requires equal error rates, regardless of protected attribute. For $Y$ the true label to $X$ and $c$ as in Eq.~\eqref{eqn:dp}, this can be written:%
\begin{equation}\small
    \mathbf{Pr}(c(X)\neq Y \mid A = 0) = \mathbf{Pr}(c(X)\neq Y  \mid A = 1).
\end{equation}%
By our definition, score parity can be used to reflect this distinct notion from classification as well. Conceptually, we select our scoring function to measure the correctness of the dialogue. Then, just like accuracy parity, score parity enforces equal error rates, regardless of protected attribute. While details may vary based on application, we consider selecting the scoring function in the examples from Figure~\ref{fig:examples}. We first define an \textbf{identifier} function $v : \mathcal{D} \to \{0,1\}$ which indicates whether a dialogue $d \in \mathcal{D}$ \textit{verbalizes} the protected attribute. For example, we can imagine $v$ scans for female gendered words $\{\text{she}, \text{her}, \text{girl}, ...\}$. Then, our system makes an ``error'' if it fails to verbalize the protected attribute or inappropriately verbalizes the attribute. So, we select the scoring function to reflect this:%
\begin{equation}\small
\label{eqn:ap_choice}
    s(D, A) = \lvert A - v(D) \rvert.
\end{equation}%
With the choice of scoring function above, score parity reflects the intuition of accuracy parity by requiring that the correctness of the language use (in referring to a protected attribute) is independent of the protected attribute.
As alluded, this constraint can be especially useful in case spurious correlations (i.e., stereotypes) between protected attributes and context cause different error rates with/without a protected attribute. This is the case in our toy examples (Figure~\ref{fig:examples}) as well as some real-world generation tasks \citep{hendricks2018women}. 
\paragraph{Takeaways} The formalization of equity we introduce -- \textit{score parity} -- is both general and useful. It models existing ideas for empirical evaluation of equity in text-generation \citep{hendricks2018women, das2022quantifying, dhamala2022analysis} and can also be used to model disparate notions of equity from existing classification theories \citep{kamiran2009classifying, calders2009building, zemel2013learning, zhao2019inherent}. Ultimately, the choice of the scoring function $s$ determines the ``way'' in which the language should be invariant to the protected attribute, and subsequently, dictates the motivating goals of the equity constraint. 
\subsection{Evaluating Equity with Learning Theory}
Next, we show how learning to generate equitable text can be modeled with learning theory.
\paragraph{Test Divergence (Reprise)}
To evaluate equity with \texttt{LEATHER}, the objective in Eq.~\eqref{eqn:TD} remains largely unchanged. Primarily, we explicitly incorporate the protected attribute:\footnote{Equivalently, one can group $A$ with the unknown effects and keep Eq.~\eqref{eqn:TD}. The rewrite only makes assumptions explicit.}%
\begin{equation}\small
\label{eqn:fairTD}
\begin{split}
    & \mathbf{TD}_\mathbb{G}(\theta) = \mathbf{E}[\lvert h(D, A, U) - h(\hat{D}, A, U) \rvert] \quad \text{where}\\ 
    & (C, D) \sim \mathbb{G}, \ \hat{D} \sim \mathbb{P}_\theta(C), \ A \sim \mathbb{A}(C), \ U \sim \mathbb{U}.
\end{split}
\end{equation}%
Importantly, we must consider the deviations from \citet{leather} \textit{not} present in Eq.~\eqref{eqn:fairTD}: (1) the choice of goal distribution $\mathbb{G}$ and (2) the choice of test $h$. Originally, \citeauthor{leather} focus on evaluation of \textit{human-like} dialogue, and therefore, propose the goal to be defined by any collected corpus of contextualized human dialogues. Instead, we are interested in the \textit{equity} of the contextualized dialogue and cannot blindly use human dialogue as an example; i.e., we cannot take for granted that the contextualized human dialogue is equitable. Thus, to appropriately evaluate equity, we generally assume the following constraints on the goal distribution and test.
\paragraph{Equitable Goals and Tests}
\begin{defn}
\label{defn:balanced}
(Balanced) A contextualized dialogue distribution $\mathbb{G}$ is \textbf{balanced} if it assigns equal (marginal) likelihood to the protected attribute:
\begin{equation}\small
    \mathbf{Pr}(A = 1) = \mathbf{Pr}(A = 0); \ (C,\cdot) \sim \mathbb{G}, \ A \sim \mathbb{A}(C).
\end{equation} 
\end{defn}
\begin{defn} 
\label{defn:eq_dist}
(Equitable Goal) We say a contextualized dialogue distribution $\mathbb{G}$ with $(C,D) \sim \mathbb{G}$ is an \textbf{equitable goal} distribution if it is balanced and satisfies score parity (for some fixed score $s$).
\end{defn}
So, intuitively, we propose the \textit{goal} in equitable dialogue is a contextualized dialogue distribution which is itself equitable, according to our formal definition of this property -- i.e., score parity. Furthermore, it should be \textit{balanced} to prioritize the protected attribute equally during evaluation.
As we'll see later, choosing the test $h$ to be the scoring function $s$ from our previous definition allows us to use $\mathbf{TD}$ (with an equitable goal) to control the parity gap of our learned dialogue.
\paragraph{Biased Data} While the formal definition above (Def.~\ref{defn:eq_dist}) is about equity, it should also be noted that we implicitly arrive at a formal definition for \textbf{bias}: \textit{the absence of equity}. In particular, a contextualized dialogue distribution (dataset) is \textbf{biased} if it is not equitable. Note, this also distinguishes biased data from other common concepts like \textit{noisy} data because we use an expectation to quantify parity; i.e., which is immune to non-systemic noise.  
\paragraph{Small Test Divergence Implies Equity}
\begin{theorem}
\label{thm:sp_reduction}
Consider an equitable goal $\mathbb{G}$ and let $h \equiv s$ (the scoring function). Then, $\Delta(\hat{\mathbb{G}}_\theta) \leq \epsilon$ whenever $\mathbf{TD}_\mathbb{G}(\theta) \leq \epsilon / 2$.
\end{theorem}
Simply, the above result indicates minimization of $\mathbf{TD}$ with an equitable goal and appropriate test leads to an equitable learned dialogue distribution. 
\paragraph{Takeaways} An important consequence of Thm. \ref{thm:sp_reduction}
is the ability to confidently use algorithms designed in the \texttt{LEATHER} framework (i.e., to reduce test divergence) for equitable dialogue learning. While these algorithms may have originally been designed to learn human-like dialogue, they can easily be modified to learn equitable dialogue. In particular, we need only change the goal from any human dialogue distribution to any equitable dialogue distribution -- as in Def.~\ref{defn:eq_dist}. Portability of algorithms in the sense described means, ultimately, a unified theory for dialogue generation. For any algorithm we propose, we may conduct a singular theoretical analysis of test divergence that can serve multiple purposes -- both human-like and equitable dialogue generation. In other words:
\begin{displayquote}
\textit{\texttt{LEATHER}-based algorithms for human-likeness can be used to learn equitable text by simply augmenting training data.}
\end{displayquote}
Some standard examples of how to create the new equitable goal $\mathbb{G}$ include augmenting data in the dataset to achieve equitable constraints \citep{zhao-etal-2018-gender, park-etal-2018-reducing}. The takeaway from our theorem above agrees with existing empirical study: we can typically expect these strategies to be effective. Still, as we see next, there are other effective alternatives (under the right assumptions). 
\subsection{Learning to be Equitable \textit{and} Human-like}
\label{sec:theory_bound}
Next, we study the circumstances under which the goals of human-like dialogue learning and equitable dialogue learning align. That is, we study circumstances under which an algorithm designed to minimize $\mathbf{TD}$ can learn from (biased) human-like goal data and simultaneously learn to be equitable.
\paragraph{Context and Its Role (Assumptions)}
\begin{defn}
\label{defn:c-awar}
(Context-Awareness) Consider an equitable goal distribution $\mathbb{G}$. A contextualized dialogue distribution $\mathbb{H} \neq \mathbb{G}$ is \textbf{context-aware} if \footnote{We use the shorthand $\mathbf{Pr}(C\vert D) = \mathbf{Pr}(\tilde{C} \vert \tilde{D})$ to mean: $\mathbf{Pr}(C = c\vert D = d) = \mathbf{Pr}(\tilde{C} = c\vert \tilde{D} = d) \ \forall \ (c,d) \in \mathcal{C} \times \mathcal{D}$.}
\begin{equation}\small
\begin{split}
\mathbf{Pr}(D \vert C) = \mathbf{Pr}(\tilde{D} \vert \tilde{C}); \  (\tilde{C},\tilde{D}) \sim \mathbb{H}, \ \tilde{A} \sim \mathbb{A}(\tilde{C}).
\end{split}
\end{equation}
\end{defn}
\begin{defn}
\label{defn:c-pres}
(Context-Preservation) The distribution $\mathbb{H}$ \textbf{preserves context} if
\begin{equation}\small
\begin{split}
\mathbf{Pr}(C \vert A) = \mathbf{Pr}(\tilde{C} \vert \tilde{A}); \  (\tilde{C},\tilde{D}) \sim \mathbb{H}, \ \tilde{A} \sim \mathbb{A}(\tilde{C}).
\end{split}
\end{equation}
\end{defn}%
The definitions are based on the idea of \textit{label-shift} used to study data-shift at test time \citep{lipton2018detecting}. In this paper, we think of $\mathbb{H}$ as the possibly inequitable distribution of \textit{human} contextualized dialogues (determined by some corpus). So, these definitions can be viewed as assumptions of how inequity presents itself in human data. 

\textit{Context-awareness} assumes that humans are not biased \textit{provided the background context $C$.} Conceptually, this is reasonable, since humans use context to form inferences about attributes of other human subjects (even protected attributes). If background is sufficient, human inferences will often be correct inferences and the dialogue should be equitable with respect to accuracy parity, at least.\footnote{Perfectly correct dialogue satisfies accuracy parity because it satisfies $s \equiv 0$ in Eq.~\eqref{eqn:ap_choice}, regardless of $A$.} Instead, bias in the considered corpus must arise from aggregate disproportions of attributes (see \S~\ref{sec:intro}). 

\textit{Context-preservation} assumes that the presentation of the context for attributes does not change. In other words, the features of the protected attribute which present themselves through the context should be invariant across $\mathbb{G}$ and $\mathbb{H}$. For example, if one attempts to infer race from an image, this assumption simply states the visual features indicative of race should be consistent. The assumption would be violated, for example, if $\mathbb{G}$ protects Asian males and $\mathbb{H}$ protects Asian females.
\paragraph{Test Divergence Learning Bound}
In this part, for simplicity, we assume the parameters $\theta$ are learned from a \textit{finite} space $\Theta$. Other proof techniques may allow arbitrary $\Theta$; e.g., \citet{maurer2004note}.
\begin{theorem}\label{thm:learning_bound}
Consider an equitable goal $\mathbb{G}$ with associated test $h$. Suppose a sample of i.i.d. human data is collected $\mathbb{S} = (\tilde{C}_i,\tilde{D}_i)_{i=1}^{m}$; $(\tilde{C}_i, \tilde{D}_i) \sim \mathbb{H}$. Suppose $\mathbb{H}$ is context aware and preserves context. Then, for all $\delta > 0$, with probability at least $1-\delta$, for all $\theta$, $2\beta \times \mathbf{TD}_\mathbb{G}(\theta)$ is bounded above by
\begin{equation}\small
\label{eqn:bound}
    \frac{1}{m}\sum_{i=1}^m \lvert \underbrace{h(\tilde{D}_i, \tilde{A}_i)}_{\text{human}} - \underbrace{h(\hat{D}'_i, \tilde{A}_i)}_{\text{predicted}} \rvert + \underbrace{\sqrt{\tfrac{\log \lvert \Theta \rvert + \ln 2 / \delta}{2m}}}_\text{data efficiency}
\end{equation}
where $\beta = \min_a \mathbf{Pr}(\tilde{A} = a)$.\footnote{Note, we also pose a technical requirement: pairwise independence must hold (conditional to the context) between the human dialogue, the predicted dialogue, and the protected attribute. This is not an overly strong assumption; see Appendix~\ref{sec:pairwise_ind} for a detailed discussion with examples.} 
\end{theorem}
For interpretation, we break down the upperbound on $2\beta \times \mathbf{TD}_\mathbb{G}(\theta)$ into two terms: (a) the difference in test output from the \textit{human} dialogue to the \textit{predicted} dialogue and (b) a \textit{data efficiency} term dependent on the number of i.i.d samples $m$. 
\paragraph{Equity from Biased Data}
Notice, the \textit{predicted} dialogue in (a) is dependent on the human dialogue's context $\tilde{C}_i$ -- not the goal dialogue's context $C$ -- so (a) is actually identical in definition to $\mathbf{TD}_\mathbb{S}$, an empirical observation of $\mathbf{TD}_\mathbb{H}$. That is, (a) is test divergence computed on a human corpus as was done by \citet{leather}. Since (a) uses a human dialogue corpus to define its goal, Eq.~\eqref{eqn:bound} implies that learning human-like dialogue (via \texttt{LEATHER}) can also optimize the equity of the dialogue by reducing an upperbound on the equitable goal $\mathbf{TD}_\mathbb{G}$. This is true even if the goal human data is biased. In other words:
\begin{displayquote}
\textit{\texttt{LEATHER}-based algorithms learn human-likeness \textbf{and} equity, even on biased data.}
\end{displayquote}
We only require the human data to be context-aware and preserve context (Defs.~\ref{defn:c-awar} and \ref{defn:c-pres}).
\paragraph{Data Efficiency} The above interpretation of (a) is only valid if the \textit{data efficiency} term (b) is also small. For interpretation, we consider the size of the parameter space $\Theta$ fixed and focus on the number of i.i.d training samples $m$. As $m$ increases, (b) ultimately goes to 0 and the effect of (a) dominates the bound. In some cases though, if $m$ is too small (b) can also have an impact. For example, this may be the case when using data-augmentation strategies to create a more equitable distribution. In particular, augmentation reduces the number of i.i.d. data points by creating dependencies in the data, which can reduce the data-efficiency of learning algorithms \citep{ralaivola2010chromatic}. That is, augmentation can increase the size of (b) in learning bounds on test divergence,\footnote{For discussion, see the pf. of Thm.~\ref{thm:learning_bound} and remarks.} or in other words:
\begin{displayquote}
\textit{Augmenting training data to improve equity can reduce data-efficiency, and ultimately, model performance.}
\end{displayquote}
Impact does depend on the augmentation strategy, so we study common proposals for equity, next.
\begin{table*}[t]
    \centering\small
    \begin{tabular}{c||c|c|c|c||c|c|c|c||c}
         &  $\mathbf{acc}$ $\uparrow$
         &  $\mathbf{l div}$ $\uparrow$ & $\mathbf{q div}$ $\uparrow$ & $\mathbf{rep q}$ & 
         $\Delta$ ($\mathbf{F}$) & $\mathbf{TD}$ ($\mathbf{F}$) &
         $\Delta$ ($\mathbf{M}$) & $\mathbf{TD}$ ($\mathbf{M})$ & $\mathbf{hum. eval.}$ ($\mathbf{F}$/$\mathbf{M}$) $\uparrow$ \\ \hline 
         \texttt{CL} & 55.9 & 10.7 & 14.3 & 58.2 
         & 52.6 & 28.8 & 23.7 & 33.5 & 52.0 / 72.0 \\
         \texttt{LEATHER} & 56.9 & 12.7 & 16.0 & 47.5 
         & 29.1 & 27.2 & 14.7 & 29.7 & 68.0 / 64.0 \\
         \texttt{DS} & 58.0 & 12.2 & 14.8 & 43.8 
         & 35.8 & 28.9 & 2.3 & 30.7 & 66.0 / 66.0
    \end{tabular}
    \caption{ \small Comparison of algorithms after 100 epochs of pre-training and 100 epochs of \textit{self-play}. Generally, objective is 0 on 100 point scale with exceptions denoted by up arrows. The first 4 metrics test human-likeness. The last 5 test equity.}
    \label{tab:results}
\end{table*}%

\section{Experiments}
\label{sec:experiments}
In Section~\ref{sec:theory}, we conclude by outlining algorithmic insights revealed by our theory. Next, we test these theories on the \textit{GuessWhat?!} game corpus. 
\subsection{Dataset, Algorithms, and Evaluation}
Unless otherwise noted, we use identical experimental settings, hyperparameters, etc. as \citet{shekhar-etal-2019-beyond, leather}. 
\paragraph{Dataset}
Our dataset is the corpus for the \textit{GuessWhat?!} game proposed by \citet{de2017guesswhat}. Gameplay is described in Figure~\ref{fig:problem} and an example is shown as the visual dialogue in Figure~\ref{fig:examples}. We also give a detailed description of the game rules in Appendix~\ref{sec:guesswhat}. We use the original train/val. splits and provide statistics on this corpus in Appendix~\ref{sec:guesswhat}. For training, unless otherwise noted, we use the full train set and report 1 seed. We focus on modelling the \textit{question-player} and use an automated answer-player trained on human data.
\paragraph{Protected Attribute} For these experiments, we use gender (male and female) as the protected attribute. When the protected attribute is female gender $(\mathbf{F})$, we set $a=1$ as long as all human dialogues use at least one female-gendered word.\footnote{\{she, woman, her, hers, gal, girl, women, gals, girls\}} When the protected attribute is male gender $(\mathbf{M})$, we set $a=1$ as long as all human dialogues use at least one male-gendered word.\footnote{\{he, man, him, his, guy, boy, men, guys, boys\}} Conceptually, this labeling scheme uses human annotator consensus to determine when it is appropriate or inappropriate to ask gender-specific questions: if $a=1$, \textit{all} human annotators perceive the protected gender to be present in the image and relevant to gameplay. Importantly, the labeling scheme also implies that the human dialogue satisfies our assumptions in \S~\ref{sec:theory_bound}: \textit{context awareness} (Def.~\ref{defn:c-awar}) and \textit{context preservation} (Def.~\ref{defn:c-pres}); i.e., as shown in Appendix~\ref{sec:label_scheme}. Different conceptualizations of how the protected attribute should be defined are possible, but we focus on this scheme because it allows us to simulate the assumptions of our theory in \S~\ref{sec:theory_bound}, and therefore, best test our theory in practice. As a final note, while we focus on male/female gender in these experiments, using more than two categories for protected attributes is also possible. Simply, one checks the parity gap for each new protected attribute to be added. This would allow our theoretical and empirical study to be extended to general multi-category attributes; e.g., race or religion. 
\paragraph{\texttt{CL} Algorithm} \texttt{CL} is a cooperative learning algorithm proposed by \citet{shekhar-etal-2019-beyond} to model the question-player. The algorithm is based primarily on a \textit{self-play} learning phase \citep{das2017learning} which learns from machine-machine dialogue. This is used in addition to (after) a more traditional supervised learning phase (i.e., on human-human dialogue). See  Appendix~\ref{sec:cl} for details.
\paragraph{\texttt{LEATHER} Algorithm} An extension of \texttt{CL} proposed by \citet{leather} with the purpose of better optimizing test divergence during the self-play learning process. Through some theoretical analyses, ultimately, the authors propose to regularize the \textit{self-play} phase by re-incorporating human-human data from the supervised phase.
\paragraph{\texttt{DS} Algorithm} A modification of the \texttt{LEATHER} algorithm. While re-incorporating human data, an augmentation (downsampling) strategy is used to balance occurrence of protected attributes; i.e., like other strategies for equity \citep{zhao-etal-2018-gender, park-etal-2018-reducing}. See Appendix~\ref{sec:downsampling} for details.
\paragraph{Human-Likeness Evaluation} To evaluate human likeness, we use metrics proposed by \citet{shekhar-etal-2019-beyond}: average accuracy $\mathbf{acc}$ in identifying the true goal-object across three random seeds, average lexical diversity ($\mathbf{ldiv}$; type/token ratio over all dialogues), average question diversity ($\mathbf{q div}$; \% unique questions over all dialogues), and average percent of dialogues with repeated questions ($\mathbf{rep q}$). We report these on the full test data.
\paragraph{Equity Evaluation} To evaluate equity, we focus on accuracy parity; i.e., score parity with scoring function described in Eq.~\eqref{eqn:ap_choice}.\footnote{We focus on accuracy parity because the dataset we consider is not likely to exhibit any significant parity issues in toxicity, sentiment, etc. Instead, the systemic biases in the data are most likely to impact accuracy parity.} To replicate evaluation against the goal distribution in Def.~\ref{defn:eq_dist}, we apply an augmentation strategy to the test set (similar to the \texttt{DS} algorithm; see Appendix~\ref{sec:downsampling}). Because our ground truth data is inferred from human annotators focused on game success, we also incorporate additional human annotations. $\mathbf{hum.eval.}$ is \% of model dialogues using gendered words correctly based on annotation (50 per method per annotator). Namely, two annotators\footnote{College educated, native English speakers.} were asked to determine correctness of gendered word use, evaluating both incorrect usage as well as false negatives; i.e., where use would be appropriate/helpful.\footnote{To prime responses, annotators were prompted with questions like ``If any gendered words were used, were they used correctly?'' as well as ``If a gendered word was not used, would it have been helpful to use one to complete the task?''.} 
\subsection{Results}
\label{sec:results}
\paragraph{\texttt{LEATHER} produces human-like, equitable text.} In Tab.~\ref{tab:results}, \texttt{LEATHER} improves upon \texttt{CL} in terms of both human-likeness \textit{and} equity, across all metrics. These observations validate our theoretical analyses. In particular, \texttt{LEATHER} (as the name implies) is designed based on the \texttt{LEATHER} framework to minimize test divergence. From previous work, we know this means it should improve human-likeness \citep{leather}. Now, from our current theoretical study (Thm.~\ref{thm:learning_bound}), we also hypothesize \texttt{LEATHER} can improve equity as long as certain assumptions are met (Def.~\ref{defn:c-awar}, \ref{defn:c-pres}). Since the dataset we study satisfies the specified assumptions, our theoretical expectation of \texttt{LEATHER} is the multi-faceted improvement we observe. That is, our theory predicts the empirical improvements in human-likeness and equity achieved by \texttt{LEATHER}. The ability of our theory to predict the impact of algorithm design choices is an important practical implication. We are also able to draw similar conclusions for \texttt{DS}, which we discuss next.
\paragraph{\texttt{DS} does not improve equity as well as \texttt{LEATHER}, but overall, its behavior aligns with our theoretical predictions.} 
Thm.~\ref{thm:learning_bound} also makes the observation that data-augmentation strategies like \texttt{DS} can sometimes perform \textit{worse} than alternatives which focus only on human-likeness (i.e., due to data-inefficiency). Since \texttt{DS} does augment data significantly, we might expect \texttt{DS} to perform worse than \texttt{LEATHER}, and ultimately, it does in Tab.~\ref{tab:results} (all metrics but $\Delta$ $\mathbf{M}$). With that said, another of our theoretical results (Thm.~\ref{thm:sp_reduction}) suggests data-augmented versions of \texttt{LEATHER} algorithms like \texttt{DS} can, in fact, improve equity, especially in more general cases where data does \textit{not} satisfy the circumstances of our experimental data. In experiments, this insight is reflected in comparing \texttt{DS} and the baseline. \texttt{DS} outperforms \texttt{CL} in Tab.~\ref{tab:results} on all metrics but $\mathbf{TD}$ $\mathbf{F}$.
\paragraph{Test divergence models equity well.} Finally, we recall test divergence is the key link between existing learning theoretic work and our analysis of equitable dialogue. In particular, we show, theoretically speaking, that $2\mathbf{TD}$ always bounds the parity gap $\Delta$, which measures equity. As a result, learning theory algorithms can implicitly learn to be fair in many cases. Indeed, empirical results in Tab.~\ref{tab:results} agree with this theoretical bound in every case, and further, suggest $\mathbf{TD}$ may be useful at ranking equity of algorithms, since $\mathbf{TD}$ is predictive of all improvements from $\texttt{CL}$ to $\texttt{LEATHER}$.  Again, our theoretical predictions match our empirical observations, highlighting the practical utilitiy of our theory.

\section{Conclusions}
\label{sec:conclusion}
In this paper, we provide a first in-depth study of equity in dialogue, formalizing mathematical notions of equity in dialogue and using computational learning theory to study how equity can be achieved through algorithm design. 
Our empirical results show how our formal theoretical study of equity in dialogue can be used, with great benefit, to select and design algorithms in a task-oriented dialogue setting. In particular, we can: design algorithms that achieve both equity and human-likeness, predict unexpected consequences of data-augmentation, and provide proxy statistics that are useful in ranking the equity of algorithms. To promote further research, our code, data, and a python package will be made publicly available.\footnote{\href{https://github.com/anthonysicilia/equitable-dialogue-ACL2023}{ https://github.com/anthonysicilia/equitable-dialogue-ACL2023}}
\section*{Acknowledgements}
The authors thank Amazon for their support during this project.
\section*{Limitations}
While our theoretical work is broadly applicable to any protected attribute and any dialogue task, our empirical study has primarily tested gender bias on the \textit{GuessWhat?!} task. Continued experimental study on a wider range of protected attributes and tasks can better support our mathematical findings. Also, users of our theory should verify the assumptions of our theory when using it to draw insights on new datasets. Specifically, as the type of data bias changes, it is possible the assumptions of Thm.~\ref{thm:learning_bound} may no longer be met. Users of our theory should take care in ensuring context-awareness and context-preservation, for example, are reasonable assumptions on new data, prior to applying the insights of \S~\ref{sec:theory_bound}. Lastly, while all of our gender annotations come from human annotators, only a smaller subset come from annotators primed to judge correctness/equity of gender reference. So, more in-depth human evaluation can better support our theoretical results as well.
\section*{Ethics Statement}
The goal of this paper is to present a theoretically grounded framework to mitigate bias in dialogue systems. Our theoretical and empirical techniques can lead to important insights/solutions for algorithm design that reduce bias, along with any unintended harm associated with this bias. With this said, some of the proposed algorithms rely on pretrained models such as word or image embeddings, and any harm or bias associated with these models can still be present after efforts to mitigate. Thus, models trained with these techniques should still undergo rigorous human evaluation for presence of biases before being deployed.

Our human subject board approved our protocol. Human subjects participated voluntarily and were compensated according to the regulations approved by our human subject review board.

\bibliography{anthology,custom}
\bibliographystyle{acl_natbib}
\clearpage
\onecolumn
\appendix
\section{Proofs and Additional Technical Discussion}
\label{sec:proofs}
\subsection{Proof of Thm.~\ref{thm:sp_reduction}}
\begin{claim}
Consider an equitable goal $\mathbb{G}$ and let $h \equiv s$ (the scoring function). Then, $\Delta(\hat{\mathbb{G}}_\theta) \leq \epsilon$ whenever $\mathbf{TD}_\mathbb{G}(\theta) \leq \epsilon / 2$.
\end{claim}
\begin{proof}
Suppose $\mathbf{TD}_\mathbb{G}(\theta) \leq \epsilon$, then we have
\begin{equation}
\begin{split}
\epsilon & \geq \mathbf{E} \big[ \big\lvert s(D, A) - s(\hat{D}, A)\big\rvert \big] \\
& = \sum_{a \in \mathcal{A}} \mathbf{Pr}(A=a) \cdot \mathbf{E}[ \big\lvert s(D, A) - s(\hat{D}, A) \big  \rvert \mid A=a] \quad \text{(Law of Total Expectation)} \\
& = \frac{1}{2} \sum_{a \in \mathcal{A}} \mathbf{E}[ \big\lvert s(D, A) - s(\hat{D}, A) \big  \rvert \mid A=a] \qquad \text{(Balance of }\mathbb{G}) \\
& \geq \frac{1}{2} \sum_{a \in \mathcal{A}} \big\lvert \mathbf{E}[ s(D, A) - s(\hat{D}, A) \mid A=a] \big  \rvert \qquad \text {(Jensen's Inequality)} \\
\end{split}
\end{equation}
Now, since $\mathbb{G}$ is equitable we have there is some value $x$ such that for all $a \in \mathcal{A}$, we have $\mathbf{E}[s(D, A) \mid A=a] = x$. Substituting and expanding the sum over $\mathcal{A}$, we have
\begin{equation}
\sum_{a \in \mathcal{A}} \big\lvert \mathbf{E}[ s(D, A) - s(\hat{D}, A) \mid A=a] \big  \rvert = \big\lvert x - \mathbf{E}[s(\hat{D}, 0)] \big \rvert + \big\lvert x - \mathbf{E}[s(\hat{D}, 1)] \big \rvert.
\end{equation}
Next, we put together the previous two equations and utilize the definition of the absolute value to break the proof into cases. For ease of presentation, we let
\begin{equation}
    \mu = \min \{\mathbf{E}[s(\hat{D}, 0)], \mathbf{E}[s(\hat{D}, 1)] \}\quad\text{and}\quad M = \max \{\mathbf{E}[s(\hat{D}, 0)], \mathbf{E}[s(\hat{D}, 1)]\}.
\end{equation}
 This gives
\begin{equation}
    2\epsilon \geq \begin{cases}
     \mathbf{E}[s(\hat{D}, 0)] - x +  \mathbf{E}[s(\hat{D}, 1)] - x & \quad\text{if}\quad \mu \geq x, \\
     x - \mathbf{E}[s(\hat{D}, 0)] + x - \mathbf{E}[s(\hat{D}, 0)] & \quad\text{if}\quad M \leq x, \\
     \mathbf{E}[s(\hat{D}, 0)] - x + x - \mathbf{E}[s(\hat{D}, 1)] & \quad\text{if}\quad \mathbf{E}[s(\hat{D}, 0)] \geq x \geq \mathbf{E}[s(\hat{D}, 1)], \\
     x - \mathbf{E}[s(\hat{D}, 0)] + \mathbf{E}[s(\hat{D}, 1)] - x& \quad\text{if}\quad \mathbf{E}[s(\hat{D}, 1)] \geq x \geq \mathbf{E}[s(\hat{D}, 0)].
    \end{cases}
\end{equation}
In the last two cases, occurrences of $x$ cancel out and we have precisely $2 \epsilon \geq \Delta(\hat{\mathbb{G}})$, precisely. Then, in the first case, we have
\begin{equation}
    \mathbf{E}[s(\hat{D}, 0)] - x +  \mathbf{E}[s(\hat{D}, 1)] - x \geq \mathbf{E}[s(\hat{D}, 0)] - \mu + \mathbf{E}[s(\hat{D}, 1)] - \mu = M - \mu.
\end{equation}
In the second case, we also have
\begin{equation}
    x - \mathbf{E}[s(\hat{D}, 0)] + x - \mathbf{E}[s(\hat{D}, 0)] \geq M - \mathbf{E}[s(\hat{D}, 0)] + M - \mathbf{E}[s(\hat{D}, 1)] = M - \mu.
\end{equation}
Thus, in all cases, we have $2\epsilon \geq \Delta(\hat{\mathbb{G}})$, the desired result.
\end{proof}
\subsection{Proof of Thm.~\ref{thm:learning_bound}}
\subsubsection{Proof}
\begin{claim}
Consider an equitable goal $\mathbb{G}$ with associated test $h$. Suppose a sample of i.i.d. human data is collected $\mathbb{S} = (\tilde{C}_i,\tilde{D}_i)_{i=1}^{m}$; $(\tilde{C}_i, \tilde{D}_i) \sim \mathbb{H}$. Suppose $\mathbb{H}$ is context aware and preserves context. Then, for all $\delta > 0$, with probability at least $1-\delta$, for all $\theta$, $2\beta \times \mathbf{TD}_\mathbb{G}(\theta)$ is bounded above by
\begin{equation}
    \frac{1}{m}\sum_{i=1}^m \lvert \underbrace{h(\tilde{D}_i, \tilde{A}_i)}_{\text{human}} - \underbrace{h(\hat{D}'_i, \tilde{A}_i)}_{\text{predicted}} \rvert + \underbrace{\sqrt{\tfrac{\log \lvert \Theta \rvert + \ln 2 / \delta}{2m}}}_\text{data efficiency}
\end{equation}
where $\beta = \min_a \mathbf{Pr}(\tilde{A} = a)$, $\hat{D}'_i \sim \mathbb{P}_\theta(\tilde{C})$. As noted in the main text we also pose the requirement of pairwise independence: first, between $D$, $\hat{D}$, and $A$ in the definition of $\mathbf{TD}_\mathbb{G}$ (conditional to $C$); second, between $\tilde{D}_i$, $\hat{D}'_i$, and $\tilde{A}_i$ (again, conditional to the context $\tilde{C}_i$). 
\end{claim}
\begin{proof}
First, we enumerate some of the key assumptions for easy reference:  
\begin{itemize}
    \item \textbf{(A1)}: $\mathbb{H}$ is context aware
    \item \textbf{(A2)}: $\mathbb{H}$ is context preserving
    \item \textbf{(A3)}: $D$, $\hat{D}$, $A$ are independent conditional to $C$; and, $\tilde{D}_i$, $\hat{D}'_i$, $\tilde{A}_i$ are independent conditional $\tilde{C}_i$
    \item \textbf{(A4)}:\footnote{Here, we are using the same shorthand from the main text; e.g., in Def.~\ref{defn:c-awar}.} $\mathbf{Pr}(\hat{D} | C) = \mathbf{Pr}(\hat{D}^\prime | \tilde{C})$ since both probabilities represent identical sampling from $\mathbb{P}_\theta$
    \item \textbf{(A5)}: $\mathbf{Pr}(A | C) = \mathbf{Pr}(\tilde{A} | \tilde{C})$ since both probabilities represent identical sampling from $\mathbb{A}$
\end{itemize}
Now, we consider decomposing the joint probability density $\mathbf{Pr}(D=d, \hat{D}=\hat{d}, A=a)$, which importantly, is the joint density used to compute the expectation in $\mathbf{TD}_\mathbb{G}(\theta)$.\footnote{We ignore $U$ since it is unused in this paper. The proof would be more complicated, but similar had we included $U$.} To begin, we have
\begin{equation}\small
\begin{split}
& \mathbf{Pr}(D=d, \hat{D}=\hat{d}, A=a) = \sum_{c} \mathbf{Pr}(C=c) \mathbf{Pr}(D=d, \hat{D}=\hat{d}, A=a \mid C = c) \quad\text{(Law of Total Exp.)} \\
& = \sum_{c} \mathbf{Pr}(C=c) \mathbf{Pr}(D=d \mid C = c)\mathbf{Pr}(\hat{D}=\hat{d} \mid C = c)\mathbf{Pr}(A=a \mid C = c) \quad\text{(\textbf{A3})} \\
& = \sum_{c} \frac{\mathbf{Pr}(C=c)}{\mathbf{Pr}(\tilde{C}=c)} \mathbf{Pr}(\tilde{C}=c) \mathbf{Pr}(D=d \mid C = c)\mathbf{Pr}(\hat{D}=\hat{d} \mid C = c)\mathbf{Pr}(A=a \mid C = c) \quad(\times\text{1 trick)} \\
& = \sum_{c} \frac{\mathbf{Pr}(C=c)}{\mathbf{Pr}(\tilde{C}=c)} \mathbf{Pr}(\tilde{C}=c) \mathbf{Pr}(\tilde{D}=d \mid \tilde{C} = c)\mathbf{Pr}(\hat{D}=\hat{d} \mid C = c)\mathbf{Pr}(A=a \mid C = c)\qquad\text{(\textbf{A1})} \\
& = \sum_{c} \frac{\mathbf{Pr}(C=c)}{\mathbf{Pr}(\tilde{C}=c)} \mathbf{Pr}(\tilde{C}=c) \mathbf{Pr}(\tilde{D}=d \mid \tilde{C} = c)\mathbf{Pr}(\hat{D}^\prime=\hat{d} \mid \tilde{C} = c)\mathbf{Pr}(A=a \mid C = c)\qquad\text{(\textbf{A4})} \\
& = \sum_{c} \frac{\mathbf{Pr}(C=c)}{\mathbf{Pr}(\tilde{C}=c)} \mathbf{Pr}(\tilde{C}=c) \mathbf{Pr}(\tilde{D}=d \mid \tilde{C} = c)\mathbf{Pr}(\hat{D}^\prime=\hat{d} \mid \tilde{C} = c)\mathbf{Pr}(\tilde{A}=a \mid \tilde{C} = c)\qquad\text{(\textbf{A5})} \\
& = \sum_{c} \frac{\mathbf{Pr}(C=c)}{\mathbf{Pr}(\tilde{C}=c)} \mathbf{Pr}(\tilde{C}=c) \mathbf{Pr}(\tilde{D}=d, \hat{D}^\prime=\hat{d}, \tilde{A}=a \mid \tilde{C}=c) \quad\text{\textbf{(A3)}}
\end{split}
\end{equation}
Further,  we can relate the probability distributions for the contexts $C$ and $\tilde{C}$ through their implied attribute distributions via \textbf{(A2)}
\begin{equation}
\begin{split}
\mathbf{Pr}(C=c) & = \sum_a \mathbf{Pr}(C = c \mid A = a) \mathbf{Pr}(A = a)\quad\text{(Law of Total Exp.)} \\
& = \sum_a \mathbf{Pr}(\tilde{C} = c \mid \tilde{A} = a) \mathbf{Pr}(A = a)\quad\text{\textbf{(A2)}} \\
& = \sum_a \mathbf{Pr}(\tilde{C} = c \mid \tilde{A} = a) \mathbf{Pr}(\tilde{A} = a) \cdot \tfrac{\mathbf{Pr}(A = a)}{\mathbf{Pr}(\tilde{A} = a)}\quad(\times\text{1 trick)} \\
& \leq \sum_a \mathbf{Pr}(\tilde{C} = c \mid \tilde{A} = a) \mathbf{Pr}(\tilde{A} = a) \cdot \tfrac{1}{2\beta}\quad\text{(balance of }\mathbb{G}\text{ and def. of }\beta) \\
& = \tfrac{1}{2\beta} \mathbf{Pr}(\tilde{C}=c)
\end{split}
\end{equation}
Applying this to our previous outcome, we have
\begin{equation}
\begin{split}
& \sum_{c} \tfrac{\mathbf{Pr}(C=c)}{\mathbf{Pr}(\tilde{C}=c)}\mathbf{Pr}(\tilde{C}=c) \mathbf{Pr}(\tilde{D}=d, \hat{D}^\prime=\hat{d}, \tilde{A}=a \mid \tilde{C}=c)\\
& \leq \sum_{c} \tfrac{1}{2\beta}\mathbf{Pr}(\tilde{C}=c) \mathbf{Pr}(\tilde{D}=d, \hat{D}^\prime=\hat{d}, \tilde{A}=a \mid \tilde{C}=c) \\
& = \tfrac{1}{2\beta} \mathbf{Pr}(\tilde{D}=d, \hat{D}^\prime=\hat{d}, \tilde{A}=a)\qquad\text{(Law of Total Exp.)}.
\end{split}
\end{equation}
Notice, the new joint density $\mathbf{Pr}(\tilde{D}=d, \hat{D}^\prime=\hat{d}, \tilde{A}=a)$ can be used to compute the expectation in $\mathbf{TD}_\mathbb{H}$, while the previous joint density was used to compute the expectation in $\mathbf{TD}_\mathbb{G}$. Both expectations have everywhere non-negative variables. So, ultimately, the relation between the joint densities gives:
\begin{equation}\label{eqn:finale}
    \mathbf{TD}_\mathbb{G}(\theta) \leq  \tfrac{1}{2\beta} \mathbf{TD}_\mathbb{H}(\theta)
\end{equation}
To complete the proof, we need to bound the true test divergence on the human data $\mathbf{TD}_\mathbb{H}(\theta)$ with our observation $\mathbf{TD}_\mathbb{S}(\theta)$. To do so, without using a test set, we need to apply a PAC learning bound for parameters selected from a finite hypothesis space (i.e., so that the result holds for any $\theta$ learned from $\Theta$). We choose the structural risk minimization bound presented in \citet{shalev2014understanding} -- i.e., Thm.~7.7 -- and apply it to our context,\footnote{To apply the theorem, we define the prefix free description language for $\Theta$ by simply enumerating each parameter in $\Theta$ (arbitrary order) and then mapping each parameter to the binary expansion of its assigned numeral. The loss needs to be replaced with the test divergence as well, but with this replacement, the required uniform convergence property for each individual parameter is still given by Hoeffding’s Inequality, so the proof as a whole is unchanged beyond this simple substitution.} which gives the final result.
\end{proof}
\subsubsection{Remarks on Data Efficiency} Note, the last step of the proof can be applied directly to $\mathbf{TD}_\mathbb{G}(\theta)$ as well, or any other instance of the test divergence for that matter. In the main text, when we refer to the data-efficiency of augmentation strategies, it is important to note that these augmentation strategies can change the distribution over which we compute test divergence. Although this distribution and the resulting test divergence may change, the data-efficiency term will be effected equally.\footnote{Some strategies for measuring data-efficiency depend on the data -- our comment excludes these.} For example, consider downsampling -- a simple augmentation strategy used in the experiments. In this case, if one downsamples to achieve balance in the frequency of the protected attribute, the data efficiency term would change from $\sqrt{\tfrac{\log \lvert \Theta \rvert + \ln 2 / \delta}{2m}}$ to $\sqrt{\tfrac{\log \lvert \Theta \rvert + \ln 2 / \delta}{2\alpha m}}$, where $\alpha$ is fraction of data remaining after downsampling. In an ideal case, where there is only one protected attribute to consider during re-balancing, we have $\alpha = 2\beta$ and the data efficiency is reduced by a factor of $1 / \sqrt{2\beta}$, compared to no augmentation. The reader may notice $\texttt{LEATHER}$ based algorithms also experience a reduction in data-efficiency by the slightly larger factor of $1 / 2\beta$ applied to the whole bound; i.e., see Eq.~\eqref{eqn:finale}. With this said, the reason we allude to worse data-efficiency overall for augmentation strategies is that these strategies typically also re-use data to define the augmentation; e.g., in the mentioned case, where one downsamples for balance, an \textit{additional} data-efficiency term must be added to the bound to measure the impact of estimating $\beta$ from training data prior to conducting the downsampling.\footnote{If this added term is $\gamma$ times the original data-efficiency, the inflation in Eq.~\eqref{eqn:finale} actually becomes \textit{smaller} than the inflation caused by data augmentation, whenever $\beta > 1 / 2 \gamma^2$.} Additional reduction can also be induced from imperfect estimation of $\beta$, and furthermore, when there is more than one protected attribute to consider. In the latter case, we may need to reduce the effective dataset size $\alpha m$ further to simulate balance (as in the later experiments; see Appendix~\ref{sec:downsampling}). Thus, depending on the problem, these compounding effects can easily lead to reduced efficiency overall; i.e., compared to basic application of $\texttt{LEATHER}$ based algorithms without augmentation on the whole dataset. Due to the complexity of this comparison, which is dependent on augmentation strategies, estimation error, etc., we leave formal comparison to future work and simply conjecture on the potential for worse data-efficiency of data augmentation strategies in the main text. Albeit, this hypothesis is confirmed in experiments throughout Section~\ref{sec:results}, and it should be noted our main argument here is that the data-efficiency of augmentation strategies needs to be considered, where it has previously not been in most literature.
\subsubsection{Assumption of Pairwise Independence}
\label{sec:pairwise_ind}
As mentioned in the main text, the assumption of pairwise independence is not an overly strong assumption. Conditional to the context $C$, pairwise independence stipulates realizations of the random values $D$, $\hat{D}$, and $A$ do not provide additional information about each other once we know $C=c$. For example, in \textit{GuessWhat?!}, knowing the gender does not impact our expectation of the QA pairs, once the image is already known. Alternatively, knowing predicted QAs does not change our expectation about human QAs, after the image is known. The latter is not so intuitive, but independence of predictions on (test) outcomes and the outcomes themselves is common among many simple learning models (e.g., fixed effects linear regression) since the learned parameters are only dependent on the i.i.d. training outcomes.
\subsection{Labeling Scheme}
\label{sec:label_scheme}
As noted, the labeling scheme for the protected attribute studied in the main text allows us to satisfy some of the key assumptions (on the human data) stipulated by Thm.~\ref{thm:learning_bound}: \textit{context awareness} (Def.~\ref{defn:c-awar}) and \textit{context preservation} (Def.~\ref{defn:c-pres}). To see this, we show that there exists an equitable goal according to score parity with scoring function defined in Eq.~\eqref{eqn:ap_choice}, and importantly, that this equitable goal is related to the human data as specified by Defs.~\ref{defn:c-awar} and \ref{defn:c-pres}. In turn, the existence of such an equitable goal implies that the human data and scoring function we study in the experiments does indeed satisfy Def.~\ref{defn:c-awar} and Def.~\ref{defn:c-pres}.

\paragraph{Construction of Goal} To begin, consider some random variables $(D, C, A)$ with the below constraints, and let $(\tilde{D}, \tilde{C}, \tilde{A})$ correspond to random variables for the human data as before. These will be used to construct the equitable goal we have just previously discussed:
\begin{equation}\label{eqn:constraints}
\begin{split}
\mathbf{Pr}(D = d \mid C = c) = \mathbf{Pr}(\tilde{D} = d \mid \tilde{C} = c), \\
\mathbf{Pr}(C = c \mid A = a) = \mathbf{Pr}(\tilde{C} = c \mid \tilde{A} = a), \\
\mathbf{Pr}(A = 0) = \mathbf{Pr}(A = 1).
\end{split}
\end{equation}
Now, also assume $D$ is independent of $A$ given $C$ (that is, \textbf{A3} in Thm.~\ref{thm:learning_bound}), so we can decompose the joint distribution of $(D, C, A)$ according to our constraints:
\begin{equation}
\begin{split}
    & \mathbf{Pr}(D=d, C=c, A=a) = \mathbf{Pr}(D=d, C=c \mid A=a) \mathbf{Pr}(A=a) \\
    & = \mathbf{Pr}(D=d \mid C=d, A=a) \mathbf{Pr}(C=c \mid A=a) \mathbf{Pr}(A=a) \\
    & = \mathbf{Pr}(D=d \mid C=c) \mathbf{Pr}(C=c \mid A=a) \mathbf{Pr}(A=a) \quad\text{(cond. indep. constraint \textbf{A3})} \\
    & = \mathbf{Pr}(\tilde{D}=d \mid \tilde{C}=c) \mathbf{Pr}(\tilde{C}=c \mid \tilde{A}=a) \mathbf{Pr}(A=a) \quad\text{(Eq.~\ref{eqn:constraints} constraints)} \\
\end{split}
\end{equation}
Next, we verify there are distributions with this joint density with total probability summing to 1. To do this, we re-use the above expansion to arrive at:
\begin{equation}\label{eqn:densitiy_simple}
\begin{split}
  & \sum_{d,c,a}  \mathbf{Pr}(D=d, C=c, A=a)
  = \sum_{d,c,a} \mathbf{Pr}(\tilde{D}=d \mid \tilde{C}=c) \mathbf{Pr}(\tilde{C}=c \mid \tilde{A}=a) \mathbf{Pr}(A=a) \\
  & = \frac{1}{2} \sum_{d,c,a} \mathbf{Pr}(\tilde{D}=d \mid \tilde{C}=c) \mathbf{Pr}(\tilde{C}=c \mid \tilde{A}=a) \quad\text{(assumed constraint on }A)\\
  & := \frac{1}{2} \Big [ x(1) + x(0) \Big ]\quad\text{(use }x(a)\text{ as a shorthand for the sum over }d,c)
\end{split}
\end{equation}
Simultaneously, since $(\tilde{D}, \tilde{C}, \tilde{A})$ already correspond to a distribution, we can use similar logic (i.e., LTE and conditional independence) to expand the sum over this distribution's joint density. In doing so, we must have
\begin{equation}
    1 = \mathbf{Pr}(\tilde{A} = 0) \cdot x(0) + \mathbf{Pr}(\tilde{A} = 1) \cdot x(1) := a \times x(1) + b \times x(0)\quad\text{(defining shorthand)}.
\end{equation}
So, the density in Eq.~\eqref{eqn:densitiy_simple} has total probability summing to 1 if there is a solution with $a,b \in [0,1]$ and $a + b = 1$ to the following system:
\begin{equation}
\begin{split}
    1 = \frac{1}{2} \Big [ x(1) + x(0) \Big ] \\
    1 = a \times x(1) + b \times x(0).
\end{split}
\end{equation}
If $a \neq b \neq 1/2$, there are solutions $a,b \in [0,1]$ with $a+b=1$ as long as $x(1) = x(0)$, which is indeed true, since due to (\textbf{A3}) $x(a)$ can be re-written as a conditional joint probability over $\tilde{D}$ and $\tilde{C}$. So, $x(1) = x(0) = 1$. Note, the other axioms of probabilities follow directly because the constraints only restrict the probabilities for $(D,C,A)$ to existing (known) probability functions. Thus, we know a distribution satisfying the needed constraints in Eq.~\eqref{eqn:constraints} exists. Specifically, a distribution related to the human data as specified by Defs.~\ref{defn:c-awar} and \ref{defn:c-pres} exists, and we have shown the desired result.

\paragraph{Equity of Goal} Finally, it remains to see how the distribution corresponding to $(D,C,A)$ is equitable. Score parity follows easily by definition of $\tilde{A} = v(\tilde{D})$. In particular, the test divergence on the human data is 0, so Eq.~\eqref{eqn:finale} implies the test divergence on the distribution of $(D,C,A)$ is 0, and so Thm.~\ref{thm:sp_reduction} implies the parity gap for the distribution of $(D,C,A)$ is 0. Balance of the distribution of $(D,C,A)$ also follows easily from the final constraint in Eq.~\eqref{eqn:constraints}, and so we are done.
\subsection{Downsampling}
\label{sec:downsampling}
The downsampling process for the \texttt{DS} algorithm restricts to images which are determined to have either of the protected attributes — i.e., $a=1$ when \textbf{M} is the protected attribute or $a=1$ when \textbf{F} is the protected attribute — such that there are an equal number of occurrences of $a=1$ for both protected attributes. That is, in the end result, the new training dataset has an equal number of occurrences where annotator consensus identified a male or a female, and all other images are thrown out. This is achieved through a simple randomized filtering approach. As noted, images without $a=1$ for either protected attribute are also thrown out. This allows us to ensure we are training a (single) model that will be equitable on both protected attributes simultaneously,\footnote{If we include images without labels, we cannot be sure of equal occurrence of both attributes.} which is the primary goal in evaluation. Note, this strategy does not hurt the object identification accuracy either (as evidenced by empirical results). This may be for two reasons: first, other objects (besides persons) appear frequently enough in the downsampled dataset as to not effect performance; second, downsampling is only used in the cooperative learning phase, and object recognition ability is primarily learned in the pre-training phase. As alluded in our theoretical discussion, another consequence of this augmentation strategy is that the number of i.i.d. data points is greatly reduced in the cooperative learning phase (e.g., compared to the \texttt{LEATHER}-based algorithm); i.e., we estimate less than 1/6th of the original dataset is used. Therefore, this indeed presents a good example to test our theoretical hypotheses on the impacts of data augmentation and data-inefficiency.

Downsampling to create the equitable distribution is done in a similar manner, except  -- since we don't need to worry about inefficiency in model training any longer -- a separate dataset is created for each protected attribute. So, there is one dataset with balanced occurrences of $a=1$ and $a=0$ when the protected attribute is \textbf{M}, and another dataset with balanced occurrences when the attribute is \textbf{F}. Importantly, because labeling scheme enforces our assumptions about context hold in the human data (see Appendix~\ref{sec:label_scheme}), this should create an equitable goal. 
\begin{figure*}
    \centering
    \includegraphics[width=0.9\linewidth]{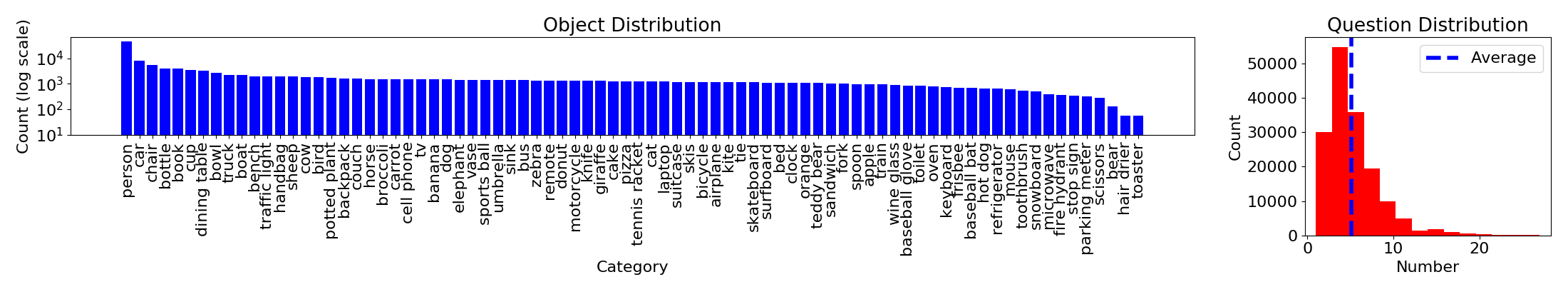}
    \caption{Statistics from the \textit{GuessWhat?!} dataset \citep{de2017guesswhat}.}
    \label{fig:object_dist}
\end{figure*}
\subsection{\textit{GuessWhat?!} Game Rules and Statistics}
\label{sec:guesswhat}
Here, we introduce the \textit{GuessWhat?!} visual dialogue game \citep{de2017guesswhat}. We use this game as a running example to ground abstract theoretical concepts in practical application. \textbf{Importantly}, our theoretical study is \textit{more generally applicable} (i.e., beyond just this example). Statistics on object distribution and dialogue length are provided in Figure~\ref{fig:object_dist}. After applying the labeling scheme and downsampling (as just described), our dataset consists of about 3200 (half with $a=1$) when \textbf{F} is the protected attribute and 6400 (half with $a=1$) when \textbf{M} is the protected attribute. Note, this also indicates that the ratio of \textbf{M} to \textbf{F} in the original dataset is about 2 to 1.
\paragraph{Gameplay} An image and \textbf{goal-object} within the image are both randomly chosen.
A \textbf{question-player} with access to the image asks yes/no questions to an \textbf{answer-player} who has access to both the image and goal-object. 
The question-player's goal is to identify the goal-object. The answer-player's goal is to reveal the goal-object to the question-player by answering the yes/no questions appropriately.
The question- and answer-player converse until the question-player is ready to make a guess or at most $m$ questions have been asked.\footnote{By default, $m=8$ following \citet{shekhar-etal-2019-beyond}.} The question-player then guesses which object was the secret goal.
\subsection{Cooperative Learning}
\label{sec:cl}
Cooperative Learning generates questions $\hat{Q}_i$ and object guess $\hat{O}$ based on answer player answers $A_i$ as below:%
\begin{equation}\small
\label{eqn:shekar_models}
\begin{split}
  & \hat{O} = \mathtt{Gues}_\alpha(\mathtt{Enc}_\beta(I, \hat{D})) \\
  & \hat{Q}_{i+1} = \mathtt{QGen}_\theta(\mathtt{Enc}_\beta(I, \hat{Q}_1, A_1, \ldots \hat{Q}_i, A_i).
\end{split}
\end{equation}%
The neural-model $\mathtt{QGen}_\theta$ is called the \textit{question-generator} and the neural-model $\mathtt{Gues}_\alpha$ is called the \textit{object-guesser}. The final neural-model $\mathtt{Enc}_\beta$ is called the \textit{encoder} and captures pertinent features for the former models to share.
All model parameters ($\alpha, \beta, \theta$) are first pre-trained on human-human dialogue and then the model-components are further updated through cooperative \textit{self-play} \citep{das2017learning}, in which the model-components and an automated answer-player play new games (machine-machine dialogue) to continue the learning process. The shared encoder is used to improve human-likeness of questions \citep{shekhar-etal-2019-beyond}.

Note, the change from Cooperative Learning (above) to Cooperative Learning with \texttt{LEATHER} simply incorporates additional human data during training the above model, instead of using only machine-machine dialogue. See \citet{leather} for more details on both approaches to cooperative learning.
\end{document}